\documentclass{article}
\usepackage{ijcai20}

\usepackage{times}

\usepackage{soul}
\usepackage{url}
\usepackage[utf8]{inputenc}
\usepackage[small]{caption}
\usepackage{graphicx,grffile}
\usepackage{amsmath}
\usepackage{amssymb}
\usepackage{amsthm}
\usepackage{booktabs}

\providecommand{\hypertarget}[2]{#2}
\makeatletter
\def\tok@scan#1{%
  \ifx#1\relax
    \let\tok@next\relax
  \else
    \edef\my@list{\my@list#1}%
    \let\tok@next\tok@scan
  \fi
  \tok@next
}
\newcommand{\@strip}[2]{%
  \def\my@list{}\tok@scan#2\relax\let#1\my@list}
\newcommand{\Cite}[1]{\@strip\@args{#1}\cite\@args}
\makeatother

\usepackage{tabularx}

\newcolumntype{Y}{>{\centering\arraybackslash}X}
\usepackage{array}

\usepackage[ruled]{algorithm2e}

\newtheorem{proposition}{Proposition}
\newtheorem{definition}{Definition}

\newtheorem{example}{Example}
\providecommand{\texorpdfstring}[2]{#1}
\usepackage{stmaryrd}
\title{On Irrelevant Literals in Pseudo-Boolean Constraint Learning}

\author{
Daniel Le Berre$^{1,2}$\and
Pierre Marquis$^{1,2,3}$\and
Stefan Mengel$^{2,1}$\And
Romain Wallon$^{1,2}$\\
\affiliations
$^1$Université d'Artois, Lens, France\\
$^2$CRIL, CNRS UMR 8188, Lens, France\\
$^3$Institut Universitaire de France\\
\emails
\{leberre, marquis, mengel, wallon\}@cril.fr
}

\begin{document}

\maketitle

\begin{abstract}
Learning pseudo-Boolean (PB) constraints in PB solvers exploiting
cutting planes based inference is not as well understood as clause
learning in conflict-driven clause learning solvers. In this paper, we
show that PB constraints derived using cutting planes may contain
\emph{irrelevant literals}, i.e., literals whose assigned values
(whatever they are) never change the truth value of the constraint. Such
literals may lead to infer constraints that are weaker than they should
be, impacting the size of the proof built by the solver, and thus also
affecting its performance. This suggests that current implementations of
PB solvers based on cutting planes should be reconsidered to prevent the
generation of irrelevant literals. Indeed, detecting and removing
irrelevant literals is too expensive in practice to be considered as an
option (the associated problem is \textsf{NP}-hard).
\end{abstract}

\hypertarget{introduction}{%
\section{Introduction}\label{introduction}}

Even though modern SAT solvers are known to perform well in practice on
many industrial benchmarks, there exist instances which remain hard to
solve, even for state-of-the-art solvers \Cite{sat12}. In particular,
this is true for unsatisfiable formulae for which inconsistency can only
be derived with an exponential number of resolution steps, e.g., for
pigeonhole-principle formulae \Cite{haken85}. Many of those hard
formulae require the solvers to be able to either detect and break
symmetries or to ``count'' so as to generate short proofs
\Cite{lakhdar2, symmetryjo, metin}. This was an important motivation
for the development of pseudo-Boolean (PB) reasoning \Cite{handbook},
which benefits from the expressiveness of \emph{PB constraints} (linear
equations or inequations over Boolean variables) and from the strength
of the \emph{cutting planes} proof system
\Cite{gomory58, hooker88, nordstrom15}. This proof system is in theory
strictly stronger than the resolution proof system used in SAT solvers,
as the former \emph{p-simulates} the latter \Cite{cook87}: any
resolution proof can be simulated by a cutting planes proof of
polynomial size w.r.t.~the size of the original proof. Yet, in practice,
none of the current PB solvers uses the full power of the cutting planes
proof system \Cite{jakobproofsolvers}. Indeed, most of them are built
on a specific form of this proof system, which can be viewed as a
generalization of resolution \Cite{hooker88}. This allows to extend
clausal inference to pseudo-Boolean inference, inheriting many of the
techniques used in SAT solving \Cite{pbchaff, galena}. In particular,
the success of the conflict-driven clause learning architecture of
modern SAT solvers \Cite{grasp, chaff, minisat} motivated its
generalization to PB problems: when a conflict is encountered (i.e.,
when a PB constraint becomes falsified), the \emph{cancellation} rule is
applied between the conflicting constraint and the reason for the
propagation of some of its literals to infer a new conflicting
constraint. Note that,~contrary to what happens for resolution based
solvers, the reason may need to be \emph{weakened} before resolving to
preserve the conflict. This operation is repeated until the inferred
constraint propagates some of its literals. The constraint is then
learned and a backjump is performed.

Over the years, many PB solvers implementing variants of the cutting
planes proof system have been developed
\Cite{pbchaff, galena, pueblo, sat4j}. Recently, \emph{RoundingSat}
\Cite{rs} introduced an aggressive use of the division and weakening
rules (see Section \ref{preliminaries} for details about cutting planes
rules). However, since the first PB evaluation \Cite{pb05}, it has been
observed that PB solvers are not as efficient as resolution based
solvers, which can solve PB problems by encoding PB constraints into
clauses \Cite{minisatp, openwbo, naps}. While PB solvers based on
cutting planes perform generally well on specific classes of benchmarks,
they fail to run uniformly well on all benchmarks
\Cite{jakobproofsolverspractical}. This is partly due to the complexity
of deciding when to use the rules of the cutting planes proof system,
and of implementing their application efficiently. The initial trend has
been to replace the application of the resolution rules by the
\emph{generalized resolution} rules \Cite{hooker88} during conflict
analysis. However, this approach is not satisfactory because it is
equivalent to resolution when applied to clauses, and requires a
specific preprocessing to derive cardinality constraints
\Cite{cardinalitydetection, jakobcard}. This is why recent years have
seen a renewed interest in PB solving and in the theory of cutting
planes based inference
\Cite{rs, jakobproofsolvers, jakobproofsolverspractical, divvssat}.

In this paper, we provide a new perspective regarding the desirable
properties of the above proof systems. In particular, we show in Section
\ref{irrelevant-literals-in-pb-constraints} that cutting planes based
inference may introduce \emph{irrelevant literals} in the derived
constraints, i.e., literals that have no effect on the truth value of
the constraint in which they appear. Such literals seem inherent to the
use of this proof system, as most of its rules may produce them, meaning
that PB solvers have either to switch back to a proof system equivalent
to the weaker resolution -- which ensures to produce no irrelevant
literals, but can only infer clauses -- or to deal with these literals.
Our main contribution is to show that, not only irrelevant literals are
produced by the derivation rules used in current PB solvers, but these
literals also contribute to the inference of constraints that are not as
strong as they should be. Unfortunately, checking whether a literal is
relevant is \textsf{NP}-complete \cite[Section~9.6]{booleanFunctions},
so in practice it would be unrealistic to remove irrelevant literals
from all PB constraints derived by a PB solver so as to infer stronger
constraints. As a workaround, we introduce in Section
\ref{eliminating-irrelevant-literals} an incomplete algorithm for
detecting and removing these literals. We use it to show in Section
\ref{experimental-results} that irrelevant literals are produced in
practice by two PB solvers, namely \emph{Sat4j} \Cite{sat4j} and
\emph{RoundingSat} \Cite{rs}, which competed in the PB evaluation 2016.
Our experiments also show that, for some instances, irrelevant literals
may have an impact on the size of the proofs built by the solver,
threatening its performance.

\hypertarget{preliminaries}{%
\section{Preliminaries}\label{preliminaries}}

We consider a propositional setting defined on a finite set of
propositional variables \(V\) which are classically interpreted. A
literal \(l\) is a propositional variable \(v \in V\) or its negation
\(\bar{v}\). In this context, Boolean values are represented by the
integers \(1\) (true) and \(0\) (false) and thus \(\bar{v} = 1 - v\).
\(\models\) denotes logical entailment and \(\equiv\) denotes logical
equivalence.

\hypertarget{pseudo-boolean-constraints}{%
\subsubsection{Pseudo-Boolean
Constraints}\label{pseudo-boolean-constraints}}

A \emph{pseudo-Boolean (PB) constraint} is an inequation of the form
\(\sum_{i = 1}^{n} \alpha_i l_i \vartriangle \delta\), where
\(\alpha_i\) and \(\delta\) are integers, \(l_i\) are literals and
\(\vartriangle \in \{\leq, <, =, >, \geq\}\). Each \(\alpha_i\) is
called a \emph{weight} or \emph{coefficient} and \(\delta\) is called
the \emph{degree} of the constraint. Any PB constraint can be
\emph{normalized} in linear time into (a conjunction of) constraints of
the form \(\sum_{i = 1}^{n} \alpha_i l_i \geq \delta\) in which the
coefficients and the degree are all positive integers. Thus, in the
following, we assume that all PB constraints are normalized. A
\emph{cardinality constraint} is a PB constraint with all its
coefficients equal to \(1\), and a \emph{clause} is a cardinality
constraint with its degree equal to \(1\).

\hypertarget{inference-rules-and-proof-systems}{%
\subsubsection{Inference Rules and Proof
Systems}\label{inference-rules-and-proof-systems}}

The PB counterpart of the well-known resolution proof system is the
\emph{cutting planes} proof system \Cite{gomory58}, which defines the
following rules \Cite{handbook}.

\hypertarget{saturation.}{%
\paragraph{Saturation.}\label{saturation.}}

The PB constraint
\(\alpha l + \sum_{i = 1}^{n} \alpha_i l_i \geq \delta\) where
\(\alpha > \delta\) is \emph{equivalent} to
\(\delta l + \sum_{i = 1}^{n} \alpha_i l_i \geq \delta\).

\hypertarget{weakening.}{%
\paragraph{Weakening.}\label{weakening.}}

The PB constraint
\(\alpha l + \sum_{i = 1}^{n} \alpha_i l_i \geq \delta\) \emph{entails}
the constraint \(\sum_{i = 1}^{n} \alpha_i l_i \geq (\delta - \alpha)\).

\hypertarget{division.}{%
\paragraph{Division.}\label{division.}}

For any integer \(\rho\), the constraint
\(\sum_{i = 1}^{n} \alpha_i l_i \geq \delta\) \emph{entails} the
constraint
\(\sum_{i = 1}^{n} \lceil{\frac{\alpha_i}{\rho}}\rceil l_i \geq \lceil{\frac{\delta}{\rho}}\rceil\).
When each \(\alpha_i\) is divisible by \(\rho\), both constraints are
\emph{equivalent}.

\hypertarget{multiplication.}{%
\paragraph{Multiplication.}\label{multiplication.}}

The PB constraint \(\sum_{i = 1}^{n} \alpha_i l_i \geq \delta\) is
\emph{equivalent} to \(\sum_{i = 1}^{n} \mu\alpha_i l_i \geq \mu\delta\)
for any integer \(\mu > 0\).

\hypertarget{addition.}{%
\paragraph{Addition.}\label{addition.}}

The conjunction of the two PB constraints
\(\sum_{i = 1}^{n} \alpha_i l_i \geq \delta\) and
\(\sum_{i = 1}^{n} \alpha'_i l'_i \geq \delta'\) \emph{entails} the sum
of both constraints, i.e.,
\(\sum_{i = 1}^{n} \alpha_i l_i + \alpha'_i l'_i \geq (\delta + \delta')\).

\hypertarget{cancellation.}{%
\paragraph{Cancellation.}\label{cancellation.}}

The conjunction of the two PB constraints
\(\alpha l + \sum_{i = 1}^{n} \alpha_i l_i \geq \delta\) and
\(\alpha \bar l + \sum_{i = 1}^{n} \alpha'_i l'_i \geq \delta'\)
\emph{entails}
\(\sum_{i = 1}^{n} \alpha_i l_i + \alpha'_i l'_i \geq (\delta + \delta' - \alpha)\).
Note that when both constraints are clauses, \emph{cancellation} is
equivalent to classical resolution \Cite{hooker88}.

\hypertarget{section}{%
\paragraph{}\label{section}}

We denote by \emph{generalized resolution} the proof system based on the
cancellation and saturation rules and by \emph{cutting planes} the proof
system allowing unrestricted linear combinations of PB constraints and
divisions. Each of these proof systems is refutationally complete
\Cite{hooker88}.

\hypertarget{irrelevant-literals-in-pb-constraints}{%
\section{Irrelevant Literals in PB
Constraints}\label{irrelevant-literals-in-pb-constraints}}

A specific problem arising with general PB constraints but not with
clauses or cardinality constraints that are neither tautological nor
contradictory is the presence of \emph{irrelevant literals}, which can
be characterized using conditioning.

\begin{definition}[Conditioning]
Given a PB constraint $\chi$ and a consistent term
$\tau$, $\chi | \tau$ is the \emph{conditioning} of $\chi$ by $\tau$,
obtained by replacing each literal in $\chi$ by $1$ if it appears in $\tau$, or
by $0$ if its opposite appears in $\tau$.
The constraint is normalized by moving constants in the left hand side
to the right hand side.
\end{definition}

\begin{definition}[Irrelevant literal]
A literal $l$ is said to be \emph{irrelevant} w.r.t.~a constraint $\chi$ when
$\chi | l \equiv \chi | \bar{l}$.
Otherwise, $l$ is said to be \emph{relevant} w.r.t.~$\chi$ (we also say that
$\chi$ \emph{depends} on $l$).
Equivalently, $l$ is \emph{irrelevant} w.r.t.~$\chi$ when flipping the value
of $l$ in any model $M$ of $\chi$ cannot make it a counter-model of~$\chi$.
\end{definition}

In the following, when there is no ambiguity about which constraint is
considered, we omit the constraint and simply say that \(l\) is relevant
or irrelevant.

The following proposition is an easy consequence of the definition of
literal relevance.

\begin{proposition}
\label{prop:min}
If there exists an irrelevant literal $l$ with coefficient $\alpha$ in a
constraint $\chi$, then all literals $l'$ having a coefficient
$\alpha' \leq \alpha$ in $\chi$ are also irrelevant.
\end{proposition}

\begin{proof}
Towards a contradiction, let $\chi$ be the constraint
$\alpha l + \alpha' l' + \sum_{i = 1}^n \alpha_i l_i \geq \delta$ with
$\alpha' \leq \alpha$ and $l$ irrelevant.
Suppose that $l'$ is relevant.
There exists a model $M$ of $\chi$ such that $M$ satisfies $l'$
and flipping its value makes $M$ a counter-model of~$\chi$.
Let us note $M'$ this counter-model.
As $l$ is irrelevant, we can suppose w.l.o.g.~that it is falsified by $M$, and
thus by $M'$.
$M$ satisfies the constraint $\chi | (\bar{l} \wedge l') \equiv
\sum_{i = 1}^{n} \alpha_i l_i \geq \delta - \alpha' \, (1)$, and so it is for
$M'$, because $M$ and $M'$ coincide on $l_i$s.
As $l$ is irrelevant, flipping its value cannot make $M'$ a model of $\chi$.
Thus, $M'$ does not satisfy $\chi | (l \wedge \bar{l'})
\equiv \sum_{i = 1}^{n} \alpha_i l_i \geq \delta - \alpha \, (2)$.
However, because $\alpha' \leq \alpha$, we have
$(1) \models (2)$, which is incompatible with the fact that $M' \models (1)$.
\end{proof}

\begin{example}
\label{ex:irrelevant}
In the constraint $10a + 5b + 5c + 2d + e + f \geq 15$, the literal $d$ is
irrelevant, and so it is for the literals $e$ and $f$.
In particular, this means that assigning these literals to any truth value
preserves the semantics of the constraint.
Thus, the three constraints $10a + 5b + 5c + 2d + e + f \geq 15$,
$10a + 5b + 5c \geq 15$ and $10a + 5b + 5c \geq 11$ are logically equivalent.
\end{example}

Observe that, even though the constraints are logically equivalent in
the previous example, they are not equivalent over the reals (in this
case, the second one is the stronger). Over the Booleans, the
\emph{slack} of a constraint \Cite{galena} is a good heuristic
indicator of its strength.

\begin{definition}[Slack]
The \emph{slack} of a PB constraint $\sum_{i = 1}^{n} \alpha_i l_i \geq \delta$
is the value $(\sum_{i = 1}^{n} \alpha_i) - \delta$.
\end{definition}

\begin{example}
\label{ex:slack}
Let us consider the constraints in Example \ref{ex:irrelevant}.
\begin{itemize}
\item $10a + 5b + 5c + 2d + e + f \geq 15$ has slack $9$.
\item $10a + 5b + 5c \geq 15$ has slack $5$.
\item $10a + 5b + 5c \geq 11$ has slack $9$.
\end{itemize}
\end{example}

This value has long been used by PB solvers, as it allows to efficiently
detect propagations and conflicts. In particular, all coefficients
having a weight greater than the slack have to be satisfied. As such,
the smaller the slack, the better the constraint from the solver
viewpoint.

\hypertarget{inference-rules-producing-irrelevant-literals}{%
\subsection{Inference Rules Producing Irrelevant
Literals}\label{inference-rules-producing-irrelevant-literals}}

Irrelevant literals seems inherent to PB reasoning. Indeed, all cutting
planes rules may infer constraints containing irrelevant literals as
long as they do not preserve equivalence, even if the constraints used
to produce them do not contain any such literals, as shown by the
following examples.

\hypertarget{weakening.-1}{%
\paragraph{Weakening.}\label{weakening.-1}}

Take the constraint \(3a + 3b + c + d \geq 4\), in which all literals
are relevant. If this contraint is weakened on~\(d\), the resulting
constraint \(3a + 3b + c \geq 3\) does not depend on \(c\) anymore.

\hypertarget{division.-1}{%
\paragraph{Division.}\label{division.-1}}

Take the constraint \(6a + 5b + c \geq 6\). All its literals are
relevant, but dividing it by \(2\) leads to the inference of the same
constraint as above, i.e., \(3a + 3b + c \geq 3\).

\hypertarget{addition.-1}{%
\paragraph{Addition.}\label{addition.-1}}

Take the two constraints \(4a + 3b + 3c \geq 6\) and
\(3b + 2a + 2d \geq 3\), which both depend on all their literals. Adding
them produces the constraint \(6a + 6b + 3c + 2d \geq 9\), in which
\(d\) is irrelevant.

\hypertarget{cancellation.-1}{%
\paragraph{Cancellation.}\label{cancellation.-1}}

If we cancel out literal \(e\) with the two constraints
\(4b + 3\bar{e} + 3c + 2a \geq 6\) and \(4a + 3e + 2b + 2d \geq 6\), in
which all literals are relevant, we again get the constraint
\(6a + 6b + 3c + 2d \geq 9\).

\hypertarget{artificially-relevant-literals-in-pb-solvers}{%
\subsection{\texorpdfstring{\emph{Artificially} Relevant Literals in PB
Solvers}{Artificially Relevant Literals in PB Solvers}}\label{artificially-relevant-literals-in-pb-solvers}}

Because the rules presented in the previous section are widely used by
PB solvers during their conflict analysis, these solvers have to deal
with constraints containing irrelevant literals. In particular, the main
issue arises when cutting planes rules are applied to these constraints:
these rules may cause irrelevant literals to become relevant in the
newly inferred constraint. When this occurs, we say that the literal has
become \emph{artificially} relevant. As we show below, this may happen
in different circumstances.

\hypertarget{generalized-resolution-based-solvers}{%
\subsubsection{Generalized Resolution Based
Solvers}\label{generalized-resolution-based-solvers}}

Let us consider a generalized resolution based PB solver, such as
\emph{Sat4j} \Cite{sat4j}. Suppose that a conflict occurs on
\(4a + 4b + 3\bar{e} + 3g + 3h + 2i + 2j \geq 16 \, (\circ)\). If the
reason for propagating \(e\) is \(6a + 6b + 4c + 3d + 3e + 2f~\geq~10\),
the conflict analysis is performed by applying the cancellation rule on
\(e\) between these two constraints. However, in some cases, the
resulting constraint may not be conflicting anymore. To preserve the
CDCL algorithm invariant, the reason of \(e\) may need to be weakened,
e.g., on \(c\). This produces the constraint
\(6a + 6b + 3d + 3e + 2f \geq 6 \, (\diamond)\), in which \(f\) is
irrelevant. Note that \((\diamond)\) has slack \(14\). Applying the
cancellation rule between the conflicting constraint \((\circ)\) and
\((\diamond)\) produces the constraint
\(10a + 10b + 3d + 3g + 3h + 2f + 2i + 2j~\geq~19 \, (\star)\), in which
\(f\) has become artificially relevant.

\hypertarget{division-based-solvers}{%
\subsubsection{Division Based Solvers}\label{division-based-solvers}}

A recent improvement in PB solving is \emph{RoundingSat} \Cite{rs},
which implements an aggressive use of the division and weakening rules
during conflict analysis.

Let us consider the constraint \(17a + 17b + 8c + 4d + 2e + 2f \geq 23\)
in which all literals are relevant. Suppose that, during the search
performed by \emph{RoundingSat}, \(c\) and~\(f\) are satisfied and all
other literals are falsified by some propagations. This constraint is
now conflictual: to analyze the conflict, \emph{RoundingSat} resolved it
against the reason for one of its falsified literals, e.g., the reason
for \(\bar{d}\). \emph{RoundingSat} weakens the constraint on \(f\), as
it is not falsified and its coefficient (\(2\)) is not divisible by the
coefficient of \(d\) (\(4\)), giving the constraint
\(17a + 17b + 8c + 4d + 2e \geq 21 \, (\nabla)\). Observe that \(e\) is
now irrelevant and that \((\nabla)\) has slack \(27\). When
\emph{RoundingSat} applies the division by \(4\), the constraint becomes
\(5a + 5b + 2c + d + e \geq 6 \, (\Delta)\), in which all literals are
relevant.

As pointed out by a reviewer, in \emph{RoundingSat}, irrelevant literals
produced after weakening a reason are always falsified by the current
assignment. Indeed, suppose that the literal \(l\) it propagates has
coefficient \(\alpha\). By construction, all remaining satisfied and
unassigned literals have a coefficient that is divisible by \(\alpha\),
and thus that is greater than \(\alpha\). As \(l\) is propagated, it is
necessarily relevant, and Proposition~\ref{prop:min} tells us that this
is also the case for these literals.

\hypertarget{impact-of-artificially-relevant-literals}{%
\subsection{Impact of Artificially Relevant
Literals}\label{impact-of-artificially-relevant-literals}}

Artificially relevant literals in the constraints inferred by the solver
may lead to infer constraints that are weaker than they could be if
irrelevant literals were not there in the first place. Indeed, remember
that, by definition, an irrelevant literal may be removed from the
constraint. This can be achieved by locally assigning it to a truth
value, as shown below.

\hypertarget{removal-by-weakening}{%
\subsubsection{Removal by Weakening}\label{removal-by-weakening}}

A first approach is assigning irrelevant literals to~\(1\), i.e.,
applying the weakening rule to these literals. This approach may
sometimes trigger the saturation rule, so that coefficients are kept
small enough. This may have an impact on the solver efficiency,
especially when arbitrary precision is required.

Let us consider the case of generalized resolution based solvers above.
If the irrelevant literal \(f\) is weakened away from \((\diamond)\),
this constraint becomes \(6a + 6b + 3d + 3e \geq 4\), which is saturated
into \(4a + 4b + 3d + 3e \geq 4 \, (\diamond_w)\), which has slack
\(10\). If this constraint is used in place of \((\diamond)\) when
applying the cancellation with \((\circ)\), one gets
\(8a + 8b + 3d + 3g + 3h + 2i + 2j \geq 17 \, (\star_w)\), which is
strictly stronger than the constraint \((\star)\) (\(a \land b\) is not
an implicant of the constraint any longer).

However, this approach does not always allow to infer stronger
constraints. Indeed, if we apply it to the \emph{RoundingSat} example,
\(e\) is weakened away from \((\nabla)\), giving the constraint
\(17a + 17b + 8c + 4d \geq 19 \, (\nabla_w)\), which has slack \(27\).
When the division by \(4\) is applied, the constraint becomes
\(5a + 5b + 2c + d \geq 5 \, (\Delta_w)\). Observe that \(c\) and \(d\)
are now irrelevant: the constraint is equivalent to \(a + b \geq 1\),
which is strictly weaker than the constraint we obtained previously.

\hypertarget{simple-removal}{%
\subsubsection{Simple Removal}\label{simple-removal}}

The second approach is assigning irrelevant literals to \(0\), i.e.,
removing them without modifying anything else on the constraint. This
approach allows to strengthen the constraint over the reals, although it
remains equivalent over the Booleans.

If we apply it to the \emph{RoundingSat} example, it indeed allows to
derive a stronger constraint: from \((\nabla)\), we now get
\(17a + 17b + 8c + 4d \geq 21 \, (\nabla_r)\), which has slack \(25\).
This constraint produces, after applying the division by \(4\), the
stronger constraint \(5a + 5b + 2c + d \geq 6 \, (\Delta_r)\).

However, considering the example of generalized resolution based
solvers, this approach removes \(f\) from \((\diamond)\), giving
\(6a + 6b + 3d + 3e \geq 6 \, (\diamond_r)\), which has slack \(12\).
When applying the cancellation rule between this constraint and
\((\circ)\), one gets
\(10a + 10b + 3d + 3g + 3h + 2i + 2j \geq 19 \, (\star_r)\), which is
stronger than \((\star)\), but weaker than \((\star_w)\). In this case,
the weakening based approach is better.

\hypertarget{slack-based-approach}{%
\subsubsection{Slack Based Approach}\label{slack-based-approach}}

The previous examples show that neither of the two removal approaches is
stronger in every situation. We thus consider a case-by-case approach
for deciding which one to choose. To this end, we consider the slack of
the PB constraint as a heuristic. Indeed, the slack is
\emph{subadditive}: given two PB constraints, the constraint obtained by
adding them has a slack that is at most equal to the sum of the slacks
of the original constraints. Minimizing the slack of the constraint to
choose will put a stronger upper bound on the slack of the constraints
that will be derived later on.

\hypertarget{eliminating-irrelevant-literals}{%
\section{Eliminating Irrelevant
Literals}\label{eliminating-irrelevant-literals}}

To evaluate the impact of irrelevant literals on PB solvers, we designed
an approach for identifying and removing them from PB constraints. Since
deciding whether a literal is relevant in a given PB constraint is
\textsf{NP}-complete \cite[Section 9.6]{booleanFunctions}, in practice,
it seems unrealistic to perform a relevance test for each literal of
each constraint derived by the solver. We thus propose an incomplete but
efficient algorithm for testing literal relevance.

For the sake of illustration, the following PB constraint \(\chi\), in
which we would like to decide whether \(l\) is relevant, will be used as
running example:

\[\alpha l + \sum_{i = 1}^{n} \alpha_i l_i \geq \delta\]

Recall that \(l\) is irrelevant if and only if
\(\chi | l \equiv \chi | \bar{l}\). Note that \(\chi | l\) and
\(\chi | \bar{l}\) only differ in the degree, and that the degree of
\(\chi | \bar{l}\) is greater. Thus, clearly
\(\chi | \bar{l} \models \chi | l\), and the equivalence test boils down
to checking whether:

\[\sum_{i = 1}^{n} \alpha_i l_i \geq \delta - \alpha
\models
\sum_{i = 1}^{n} \alpha_i l_i \geq \delta\]

Observe that this statement holds if and only if there is no
interpretation of \(\sum_{i = 1}^{n} \alpha_i l_i\) equal to any number
between \(\delta - \alpha\) and \(\delta - 1\). Thus, checking that
\(l\) is irrelevant is equivalent to checking that there is no subset of
\(\alpha_1,...,\alpha_n\) whose sum equals any of these numbers, i.e.,
solving an instance of the subset-sum problem for each of these inputs.

It is folklore that this can be done in time \(O(n\delta)\) using
dynamic programming \cite[Chapter~34.5]{algorithms} that is
pseudopolynomial in the encoding size. However, in our context, both
\(n\) and \(\delta\) may be very large, and it would be very inefficient
to solve subset-sum on such inputs. As a workaround, we present an
approach for solving subset-sum \emph{incompletely}. Our detection
algorithm needs to ensure that there is \emph{no} solution to the
considered subset-sum instance in order to correctly detect irrelevant
literals, even though some of them may be missed. To this end, we
introduce a detection algorithm based on solving subset-sum
\emph{modulo} a given positive integer \(p\) (fixed for all applications
of this algorithm). Since modular arithmetic is compatible with
addition, one can ensure that, if there is a solution for the subset-sum
problem with the original values, this solution is also a solution of
the subset-sum problem considered modulo \(p\).

As this procedure remains time consuming, one can also take advantage of
Proposition~\ref{prop:min} to reduce the number of checks to perform.
This can be achieved by ordering the literals by ascending coefficients.
Only one check per coefficient is required, and once a relevant literal
is identified, there is no more irrelevant literal to remove.

\begin{example}
\label{ex1}
Take the constraint $12a + 6b + 6c + 2d + 2e~\geq~18$.
If we want to check the relevance of $e$ in this constraint, the multiset of
coefficients to consider is $\{ 12, 6, 6, 2\}$ (as $e$ is ignored for the
purpose of the check).

First, let us consider $p = 5$.
The multiset of coefficients modulo $p$ is $\{ 2, 1, 1, 2 \}$.
The set of all possible subset sums modulo $5$ is thus $\{ 0, 1, 2, 3, 4 \}$, and $e$ is wrongly detected as relevant, since there exists a subset sum equal to $2 \equiv 17 \mod 5$.
If we now consider $p = 6$, the multiset of coefficients becomes
$\{ 0, 0, 0, 2 \}$, and the possible subset sums modulo $6$ are $\{ 0, 2 \}$.
It is thus impossible to find any sum equal to either $4 \equiv 16 \mod 6$
or $5 \equiv 17 \mod 6$, so $e$ is irrelevant.

As a consequence, $d$ can also be removed, since it has the same coefficient as $e$.
Then, as $c$ is relevant, it is detected as such by our algorithm, which never gives the wrong answer for
relevant literals.
All remaining literals are thus relevant and the removal stops.
\end{example}

\hypertarget{experimental-results}{%
\section{Experimental Results}\label{experimental-results}}

This section provides experimental results showing to what extent
irrelevant literals are present in the constraints inferred by PB
solvers. We also give an attempt to evaluate their impact on the
performance.

To do so, we implemented the incomplete detection algorithm presented
above in \emph{Sat4j} \Cite{sat4j}, and consider the whole set of
decision benchmarks containing only small integers used in the PB
evaluation since its very first edition \Cite{pb05}. After some
preliminary experiments, we chose \(4547\) as parameter \(p\) of our
incomplete subset-sum algorithm, and \(500\) as bound on the number of
literals in the constraints to consider. These values have been chosen
in a way that nearly all constraints in our experiments are treated
efficiently, while the number of literals wrongly detected as relevant
remains reasonable.

All experiments presented in this section have been run on a cluster
equipped with quadcore bi-processors Intel XEON E5-5637 v4 (3.5 GHz) and
128 GB of memory, with a memory limit set to 64~GB.

\hypertarget{production-of-irrelevant-literals}{%
\subsection{Production of Irrelevant
literals}\label{production-of-irrelevant-literals}}

In order to see to what extent PB solvers produce irrelevant literals,
we first use the set of benchmarks as input for both \emph{Sat4j}'s
implementation of generalized resolution (\emph{Sat4j-CP}) \Cite{sat4j}
and \emph{RoundingSat} \Cite{rs}, without modifying their conflict
analysis. We chose these two solvers because, in the latest competition,
no other solver implemented cutting planes based inference to solve PB
instances (note that \emph{RoundingSat} was called
\emph{cdcl-cuttingplanes} at that time). For these experiments, we let
these solvers run for 5 minutes, and let them dump at most 100,000
constraints derived during their conflict analyses, which are then given
to the detection algorithm we implemented in \emph{Sat4j}. The
constraints are dumped after the application of any rule that may
introduce irrelevant literals. In the case of \emph{Sat4j-CP},
irrelevant literals may be produced either after having applied the
weakening operation on the reason, or after having applied the
cancellation rule between the reason and the conflicting constraints.
Regarding \emph{RoundingSat}, irrelevant literals are always ``hidden''.
Indeed, for efficiency reasons, the weakening and division rules are
applied at the same time (we decoupled these two operations for our
experiments, as time was not considered). However, by construction, if
the weakening operation produces irrelevant literals, the following
division, by ensuring that the pivot for the cancellation rule has a
weight equal to \(1\), will make all irrelevant literals artificially
relevant (the pivot can never be irrelevant, and so do literals sharing
the same coefficient at the end). As such, irrelevant literals produced
in \emph{RoundingSat} become systematically artificially relevant at the
same step.

\begin{figure}[h]
    \includegraphics[width=.45\textwidth]{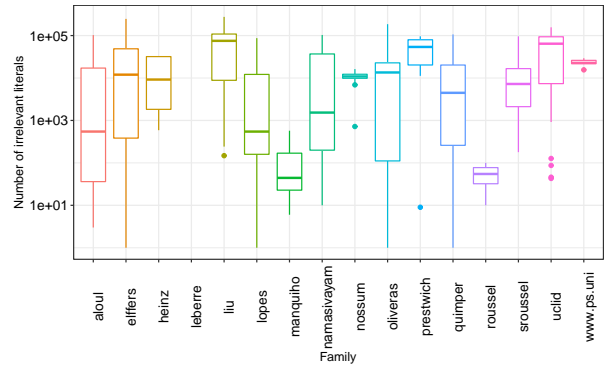}
    \caption{
        Number of irrelevant literals produced by \emph{Sat4j-CP}.
        For readability, only submitters' names are displayed.
    }
    \label{sat4j-preliminary}
\end{figure}

\begin{figure}[h]
    \includegraphics[width=.45\textwidth]{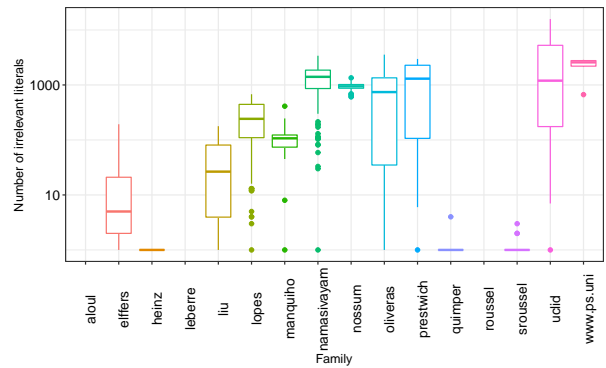}
    \caption{
        Number of irrelevant literals produced by \emph{RoundingSat}.
        For readability, only submitters' names are displayed.
    }
    \label{roundingsat-preliminary}
\end{figure}

The results of these experiments are shown in Figures
\ref{sat4j-preliminary} and~\ref{roundingsat-preliminary} with
\emph{boxplots} illustrating how many irrelevant literals were produced
by the solver for each family. Each boxplot displays the quartiles with
the horizontal bars and the estimated minimum and maximum with the
vertical bars computed from the number of detected irrelevant literals
in each instance of this family. Points represent \emph{outliers}, which
are instances for which the number of detected irrelevant literals is
either below or above the estimated minimum or maximum, respectively.
Because there are big differences between the families, boxplots are
drawn using logarithmic scales. Also, the number of constraints in which
irrelevant literals appear varies between the families: in some
families, only few constraints contain many irrelevant literals, whereas
in some others up to 75\% of the constraints contain irrelevant
literals. These boxplots reveal that \emph{Sat4j} produces irrelevant
literals (see Figure \ref{sat4j-preliminary}), and that this is also the
case for \emph{RoundingSat}, but to a lesser extent (see Figure
\ref{roundingsat-preliminary}). However, because all irrelevant literals
produced by \emph{RoundingSat} become immediately artificially relevant
literals, they appear as irrelevant only once, while irrelevant literals
produced by \emph{Sat4j} during conflict analysis may remain irrelevant
in several consecutive derivation steps, and thus appear as such
multiple times.

\hypertarget{removal-of-irrelevant-literals}{%
\subsection{Removal of Irrelevant
Literals}\label{removal-of-irrelevant-literals}}

As we showed in a previous section, the presence of irrelevant literals
in the constraints that are derived during conflict analysis may lead to
the inference of weaker constraints. If we want to avoid this behavior,
we need to remove \emph{all} irrelevant literals produced during this
process. Since this task is \textsf{NP}-hard, we used the incomplete
approach (described above) for detecting these literals. However, this
approach remains costly in practice, as 26\% of the runtime is spent in
average to detect irrelevant literals. For some families, such as
\texttt{Aardal\_1}, \texttt{armies}, \texttt{ShortestPathBA} and
\texttt{tsp}, more than 75\% of the runtime is required for most of the
instances. This is why we considered a time independent measure to
evaluate the impact of irrelevant literals: for all instances, we
compared the sizes of the proofs built by the solver, measured in number
of generalized resolution steps applied during the whole execution of
the solver. Only unsatisfiable instances are considered here, as the
size of the proofs only makes sense for such instances. The results are
shown on Figure \ref{proof-size}.

\begin{figure}
    \centering
    \includegraphics[width=.48\textwidth]{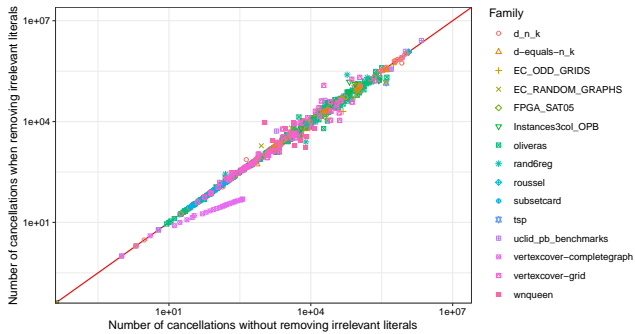}
    \caption{
        Comparison of the size of the proof built by \emph{Sat4j-CP} on all
        unsatisfiable instances.
        The families used to describe the instances are the most specific
        subfamilies.
    }
    \label{proof-size}
\end{figure}

At first sight, it seems that there is not a great difference between
the two approaches (this is also the case between the three different
approaches presented in this paper for removing irrelevant literals,
omitted for space reason). However, remember that our approach is
incomplete and that, even though nearly all constraints are treated in
most families (less than 75\% of the constraints are treated only for
instances from \texttt{d-equals-n\_k}, \texttt{FPGA\_SAT05},
\texttt{robin}, \texttt{ShortestPathBA} and \texttt{ShortestPathNG}), we
may wrongly detect as relevant literals that are actually irrelevant. To
really evaluate the impact of irrelevant literals on the solver
performance, we would need to remove all of them. However, in practice,
a complete approach is clearly unreasonable: our algorithm manages to
deal with constraints having a degree up to \(10^{410415}\) which is out
of reach of any complete approach.

Also, note that removing irrelevant literals impacts the heuristic used
to select the variables to assign. Indeed, as for classical SAT solving,
PB solvers use a heuristic based on VSIDS \Cite{chaff}. This heuristic
\emph{bumps} the variables that are involved in a conflict, i.e., makes
them more likely to be chosen next. When irrelevant literals are
removed, the associated variables are not bumped anymore, which alters
the behavior of the heuristic, and may have unexpected side effects as
this heuristic is not fully understood \Cite{cdclpractice}. In
particular, it is hard to evaluate the impact of bumping irrelevant
literals. On the one hand, one could argue that, because these literals
are irrelevant, they do not play any role in the conflict. On the other
hand, if they were relevant at some point, then their assignment may
have triggered some propagations, and, in such a case, they may actually
have contributed to the falsification of the constraint.

However, our contribution is not to find a solution to the problem of
irrelevant literals, but to get a better understanding of their impact.
To this end, let us consider more specifically the
\texttt{vertexcover-completegraph} family for which
Figure~\ref{proof-size} shows that elimination of irrelevant literals
has a significant impact on the size of the proof produced by
\emph{Sat4j}. The instances of this family encode that complete graphs
do not have small vertex covers \Cite{jakobproofsolverspractical}. As
shown by Figure \ref{proof-size}, the number of performed cancellations
is exponentially smaller after removing irrelevant literals. A closer
inspection at the solver's behavior shows that only few irrelevant
literals are actually removed during the search. In particular, all
these literals are detected and removed after the first conflict
analysis, which produces a constraint of the form
\(kx_1 + x_2 + ... + x_{k} \geq k\), where
\(k = \lceil \frac{n}{2} \rceil - 1\). One can observe that
\(x_2, ..., x_k\) are all irrelevant because their coefficients sum up
to only \(k-1\), and that the constraint is actually equivalent to the
unit clause \(x_1 \geq 1\). In all further conflict analyses, no
irrelevant literals are produced: this illustrates how few irrelevant
literals may have an impact on the whole proof built by the solver, and
may threaten its performance.

\hypertarget{conclusion-and-future-works}{%
\section{Conclusion and Future
Works}\label{conclusion-and-future-works}}

In this paper, we have shown that irrelevant literals may be introduced
in the constraints derived by PB solvers using cutting planes rules, and
that such literals may have an impact on the strength of the reasoning.
In particular, when irrelevant literals appear in the intermediate
constraints, the learned constraint may be weaker than it could be, as
it may contain artificially relevant literals. This may even happen when
the constraint is a clause or a cardinality constraint, despite the fact
that such constraints cannot contain irrelevant literals. We emphasized
that while assigning irrelevant literals produces logically equivalent
constraints, their slack may differ. The slack provides thus a
convenient heuristic to handle irrelevant literals in a solver. To
evaluate the practical impact of these literals on PB solvers, we
designed an approximation algorithm for detecting and removing them at
each derivation step. Our experimental results show that this approach
allows to find irrelevant literals in PB solvers such as \emph{Sat4j}
and \emph{RoundingSat}, and that these literals may have an impact on
the size of the proof they build.

Our approach for eliminating irrelevant literals is however too costly
in practice to be considered as a counter-measure to their production in
current PB solvers. A possible improvement for our algorithm, as
suggested by one of the reviewers, is to consider multiple subset sum
problems with small prime numbers, instead of one problem with a single
large number, as for the Chinese remainder theorem. However, the best
approach would be to avoid introducing irrelevant literals. Our ultimate
goal is to define a proof system on PB constraints which ensures that
constraints derived from PB constraints with only relevant literals do
not contain irrelevant literals. The main difficulty is to find a
\emph{complete} set of rules that can be efficiently implemented to
perform conflict analysis, and replace the current approaches used in PB
solvers.

\section*{Acknowledgements}

The authors are grateful to the anonymous reviewers for their numerous
comments, that greatly helped to improve the presentation of the paper.
Part of this work was supported by the French Ministry for Higher
Education and Research and the Hauts-de-France Regional Council through
the ``Contrat de Plan État Région (CPER) DATA''.

\bibliographystyle{named}
\bibliography{bibliography}
\end{document}